\documentclass{cstr}

\usepackage{amsmath,amsfonts,amssymb,latexsym,bm}
\usepackage{amsthm}
\usepackage{mathrsfs}
\usepackage[pdftex]{graphicx}
\usepackage{subfig}
\usepackage{url}
\usepackage{cite}

\usepackage{booktabs, multirow}

\newtheorem{definition}{Definition}

\newtheorem{proposition}{Proposition}

\usepackage{algorithm, algorithmic}
\renewcommand{\algorithmicrequire}{{\bf Input (for worker-side):}}
\renewcommand{\algorithmicensure}{{\bf Output (for worker-side):}}

\title{
  Non-readily identifiable data collaboration analysis for multiple datasets including personal information
}

\author[1,*]{Akira Imakura}
\author[1]{Tetsuya Sakurai}
\author[1]{Yukihiko Okada}
\author[2]{Tomoya Fujii}
\author[2]{Teppei Sakamoto}
\author[2]{Hiroyuki Abe}

\affil[1]{University of Tsukuba, 1-1-1 Tennodai, Ibaraki, Tsukuba 305-8573, Japan}
\affil[2]{NTT DATA Corporation, Toyosu Center Building, 3-3, Toyosu 3-chome, Koto-ku, Tokyo 135-6033, Japan}

\email{imakura@cs.tsukuba.ac.jp}

\begin{document}
\maketitle

\begin{abstract}
Multi-source data fusion, in which multiple data sources are jointly analyzed to obtain improved information, has considerable research attention.
For the datasets of multiple medical institutions, data confidentiality and cross-institutional communication are critical.
In such cases, data collaboration (DC) analysis by sharing dimensionality-reduced intermediate representations without iterative cross-institutional communications may be appropriate.
{\it Identifiability} of the shared data is essential when analyzing data including personal information.
In this study, the identifiability of the DC analysis is investigated.
The results reveals that the shared intermediate representations are readily identifiable to the original data for supervised learning.
This study then proposes a non-readily identifiable DC analysis only sharing non-readily identifiable data for multiple medical datasets including personal information.
The proposed method solves identifiability concerns based on a random sample permutation, the concept of interpretable DC analysis, and usage of functions that cannot be reconstructed.
In numerical experiments on medical datasets, the proposed method exhibits a non-readily identifiability while maintaining a high recognition performance of the conventional DC analysis.
For a hospital dataset, the proposed method exhibits a nine percentage point improvement regarding the recognition performance over the local analysis that uses only local dataset.
\end{abstract}

\section{Introduction}
\subsection{Background}
Multi-source data fusion, in which multiple data sources are jointly analyzed to obtain improved information or refined data with lower cost, higher quality, and more relevant information \footnote{This is a definition introduced in \cite{zhang2021tensor}. Because ``data fusion'' is a broad concept, diverse definitions exist.}, have attracted considerable research attention \cite{acar2014structure,wang2019data,zhang2021tensor,wan2022uav,zhang2022data}.
For example, in medical data analysis for rare diseases, it was reported that when the analysis is conducted using only data from a single institution, the accuracy is insufficient because of the small sample size \cite{mascalzoni2014rare}.
On the other hand, in some real-world applications, such as medical, financial, and manufacturing data analyses, sharing the original data for analysis is difficult because of data confidentiality.
Furthermore, privacy-preserving analysis methods, in which datasets are collaboratively analyzed without sharing the original data, are essential.
\par
Federated learning \cite{li2019survey,konevcny2016federated,mcmahan2016communication,yang2019federated,criado2022non-iid} is a typical technology for this topic.
In multi-source data fusion on multiple medial institutions, cross-institutional communication is a major concern for social implementation.
In such cases, instead of the federated learnings based on iterative model updating with cross-institutional communications, data collaboration (DC) analysis \cite{imakura2020data,imakura2021collaborative} by sharing intermediate representations constructed by some dimensionality reduction method such as principal component analysis (PCA) \cite{pearson1901liii} without iterative cross-institutional communications is preferred.
\par
{\it Identifiability} of the shared data is essential for analyzing data including personal information.
In the general data protection regulation (GDPR) of EU, California consumer privacy act (CCPA) of USA, and amended act on the protection of personal information (APPI) of Japan, data that can indirectly identify individuals are defined as personal information.
\par
In this study, a mathematical definition of the identifiability of the data is introduced; see Definition~1 in Section~3.
The shared data should be non-readily identifiable in privacy-preserving machine learning.
\subsection{Motivation and contributions}
\begin{figure*}[!t]
\centering
\includegraphics[scale=0.35, bb = 0 0 1134 567]{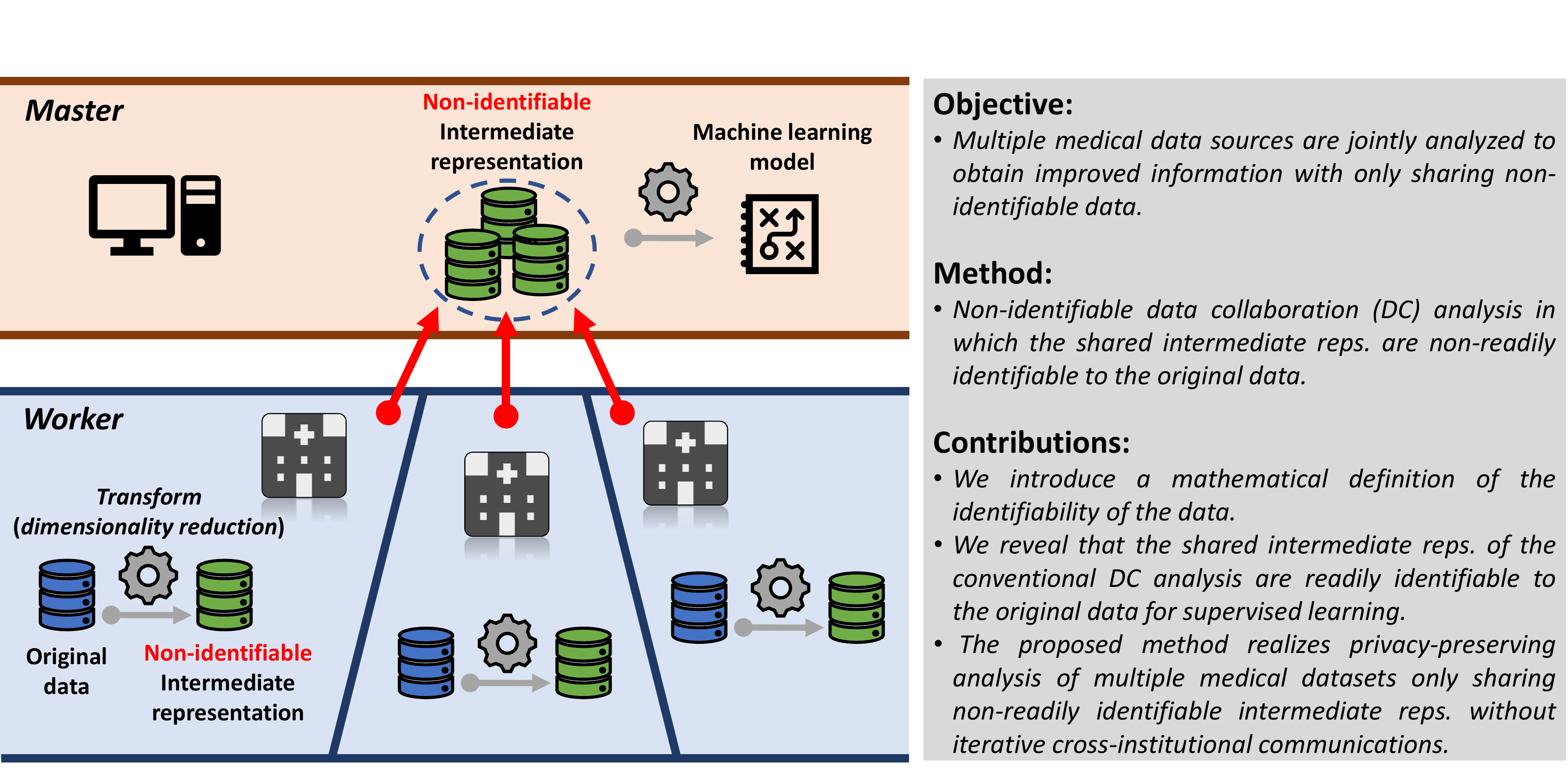}
\caption{
  Outline of the proposed method.
  The proposed method realizes privacy-preserving analysis of multiple medical datasets sharing only non-readily identifiable intermediate representations without iterative cross-institutional communications.
}
\label{fig:outline}
\end{figure*}
In this study, we focus on supervised machine learning for classification problems of medical data held by multiple medial institutions and investigate the identifiability of the DC analysis.
We analyze the identifiability of the shared intermediate representations and reveal that the shared intermediate representations are readily identifiable to the original data for supervised learning even with a random sample permutation.
Next, we propose a non-readily identifiable DC analysis only sharing non-readily identifiable data based on strategies: a random sample permutation, the concept of interpretable DC analysis \cite{imakura2021interpretable}, and usage of functions that cannot be reconstructed (Figure~\ref{fig:outline}).
\par
The main contributions are as follows:
\begin{itemize}
  \item This study introduces a mathematical definition of the identifiability of the data (Definition~1 in Section~3), which is based on whether or not someone hold a key to collate non-personal information with personal information.
  \item This study reveals that the shared intermediate representations of the conventional DC analysis are readily identifiable to the original data for supervised learning even with a random sample permutation (Section 4.2).
  \item This study proposes a non-readily identifiable DC analysis that realizes privacy-preserving analysis of multiple medical datasets sharing only non-readily identifiable intermediate representations without iterative cross-institutional communications (Section~5).
\end{itemize}
In numerical experiments in Section~6 on medical datasets, the proposed non-readily identifiable DC analysis exhibits a non-readily identifiability while maintaining a high recognition performance of the conventional DC analysis for medical datasets (Section~6).
The numerical experiments also demonstrates that the recognition performance of the proposed method is comparable to the centralized analysis that shares the raw datasets.
\section{Related works}
\subsection{Federated Learning}
Federated learning systems have been developed for privacy-preserving analysis on multiple datasets.
The concept of federated learning was first proposed by Google \cite{konevcny2016federated} typically for Android phone model updates \cite{mcmahan2016communication}.
Federated learning is primarily based on (deep) neural network and updates the model iteratively \cite{li2019survey,konevcny2016federated,mcmahan2016communication,yang2019federated}.
\par
To update the model, federated stochastic gradient descent (FedSGD) and federated averaging (FedAvg) are typical strategies \cite{mcmahan2016communication}. 
Federated learning including more recent methods, such as FedProx \cite{li2020federated} and FedCodl \cite{ni2022federated}, requires cross-institutional communication in each iteration.
This requirement is a major concern for social implementation specifically for data owned by multiple medical institutions.
For more details, we refer to \cite{li2019survey,yang2019federated} and references therein.
\subsection{DC analysis}
As an approach for privacy-preserving analysis on multiple datasets, non-model share-type federated learning called {\it DC analysis} has been proposed for supervised learning \cite{imakura2020data,imakura2021collaborative,imakura2021interpretable,mizoguchi2022application}, novelty detection \cite{imakura2021collaborative2}, and feature selection \cite{ye2019distributed}.
In DC analysis, the dimensionality-reduced {\it intermediate representations} are centralized instead of sharing the model.
The centralized intermediate representations are transformed to incorporable forms called {\it collaboration representations}.
For constructing the incorporable collaboration representations, all the parties generate a shareable {\it anchor dataset} and centralize its intermediate representation.
Next, the collaborative representation is analyzed as a single dataset.
\par
DC analysis preserves the privacy of the original data by allowing each party to use individual functions to generate the intermediate representation and not sharing them \cite{imakura2021accuracy}.
The DC analysis does not require iterative communications between parties.
\subsection{Homomorphic encryption computation}
Cryptographic computation is one of the most well-known methods used for ensuring privacy preservation \cite{jha2005privacy,cho2018secure,gilad2016cryptonets}.
Cryptographic methods can compute a function over distributed data while retaining the privacy of the data.
Any given function can be computed by applying fully homomorphic encryption \cite{gentry2009fully}.
However, this method is not feasible for large datasets because of the large computational cost even with the latest implementations \cite{chillotti2016faster,zalonis2022report}.
\section{Identifiability of the data}
\subsection{Importance of identifiability}
When analyzing data including personal information, various privacy and data protection laws, professional duties, and custodial obligations are applied.
If the analysis is outsourced to another organization, the analysis organization is subject to the same obligations.
In addition to directly personally identifiable data, similar obligations could be imposed on data that can be linked with personal information and thereby indirectly identify individuals.
\par
For example, in laws regarding personal information in countries such as the general data protection regulation (GDPR) of EU, California consumer privacy act (CCPA) of USA, and amended act on the protection of personal information (APPI) of Japan, data that indirectly identifies individuals are defined as personal information as follows:
\begin{itemize}
  \item {\bf Article 4(1) of GDPR}: \\
    {\it ``Personal data'' means any information relating to an identified or identifiable natural person (`data subject'); an identifiable natural person is one who can be identified, directly or indirectly.}
  \item {\bf 1798.140(o) of CCPA}: \\ 
    {\it ``Personal information'' means information that identifies, relates to, describes, is capable of being associated with, or could reasonably be linked, directly or indirectly, with a particular consumer or household.}
  \item {\bf Article 2 of Amended APPI}: \\
    {\it ``Personal information'' in this Act means that information relating to a living individual \dots (including those which can be readily collated with other information and thereby identify a specific individual).}
\end{itemize}
\par
Therefore, the identifiability of the shared information is essential for privacy-preserving analysis.
Because readily identifiable data are treated in the same manner as personal information data, the shared data should be non-readily identifiable to the original data for privacy-preserving analysis on multiple datasets including personal information.
\subsection{Definition and property of identifiability}
In this study, we define {\it identifiability} of the data as the data readily collated with the original data including personal information.
Here, we introduce the following mathematical definition of readily identifiable data.
\begin{definition}
  Let $x_i^{\rm p}$ and $x_i^{\rm np}$ be a pair of the data including and not including personal information that can directly identify a specific individual for the $i$-th person, respectively.
  We let $\mathcal{X}^{\rm p} = \{x_1^{\rm p}, x_2^{\rm p}, \dots, x_n^{\rm p}\}$ and $\mathcal{X}^{\rm np} = \{x_1^{\rm np}, x_2^{\rm np}, \dots, x_n^{\rm np}\}$ be personal information and non-personal information datasets for the same $n$ persons, respectively.
  \par
  For non-personal information $x^{\rm np} \in \mathcal{X}^{\rm np}$, if and only if someone else holds a key to correctly collate the corresponding personal information $x^{\rm p} \in \mathcal{X}^{\rm p}$ or can generate the key by their own, then the non-personal dataset $\mathcal{X}^{\rm np}$ is defined as ``readily identifiable'' to personal dataset $\mathcal{X}^{\rm p}$.
\end{definition}
Examples of the key to collate the data are unique common IDs and features.
In addition, excellent approximations of the features can be a key to correctly collate the data.
Here, we have the following property.
\begin{proposition}
  If either of the followings:
  \begin{itemize}
    \item The data holder of $\mathcal{X}^{\rm p}$ holds the function $v$ such that $x_i^{\rm np}=v(x_i^{\rm p})$ or can generate the function by their own,
    \item The data holder of $\mathcal{X}^{\rm np}$ holds the function $w$ such that $x_i^{\rm p}=w(x_i^{\rm np})$ or can generate the function by their own,
  \end{itemize}
  then $\mathcal{X}^{\rm np}$ is readily identifiable to $\mathcal{X}^{\rm p}$.
\end{proposition}
\begin{proof}
  In case, the data holder of $\mathcal{X}^{\rm p}$ holds the function $v$ such that $x_i^{\rm np}=v(x_i^{\rm p})$ or can generate the function by their own, they can obtain pairs of $(x_i^{\rm p},x_i^{\rm np})$ corresponding to any $x_i^{\rm p} \in \mathcal{X}^{\rm p}$ using the function $v$.
  Therefore, using $x_i^{\rm np}$ as a key, for non-personal information $x^{\rm np} \in \mathcal{X}^{\rm np}$, the corresponding personal information $x^{\rm p} \in \mathcal{X}^{\rm p}$ can be collated accurately.
  \par
  In the same manner, for the case that the data holder of $\mathcal{X}^{\rm np}$ holds the function $w$ such that $x_i^{\rm p}=w(x_i^{\rm np})$ or can generate the function by their own, they can obtain pairs of $(x_i^{\rm p},x_i^{\rm np})$ corresponding to any $x_i^{\rm np} \in \mathcal{X}^{\rm np}$ using the function $w$.
  Therefore, using $x_i^{\rm p}$ as a key, for non-personal information $x^{\rm np} \in \mathcal{X}^{\rm np}$, the corresponding personal information $x^{\rm p} \in \mathcal{X}^{\rm p}$ can be collated accurately.
\end{proof}
Proposition~1 indicates that encrypted datasets shared in homomorphic encryption computation are readily identifiable to the original data because someone holds encryption and decryption functions.
Retention of excellent approximations of functions $v$ and $w$ can result in readily identifiable.
\section{Identifiability of DC analysis}
In this study, we focus on supervised machine learning for classification problems, which aims to construct a prediction or classification model by labeled training datasets \cite{bishop2006pattern}, of medical data of multiple medial institutions.
\par
Let $m$ and $n$ denote the numbers of features (dimensionality of each data) and training data samples.
Let $X = [{\bm x}_{1}, {\bm x}_{2}, \dots, {\bm x}_{n}]^{\rm T} \in \mathbb{R}^{n \times m}$ and $Y = [{\bm y}_1, {\bm y}_2, \dots, {\bm y}_n]^{\rm T} \in \mathbb{R}^{n \times \ell}$ be the training dataset and the corresponding ground truth or label.
In this study, for privacy-preserving analysis on multiple parties, we consider horizontal data partitioning, that is, data samples are partitioned into $c$ parties as follows:
\begin{equation}
  X = \left[
    \begin{array}{c}
      X_{1} \\
      X_{2} \\
      \vdots  \\
      X_{c} 
    \end{array}
  \right], \quad
  Y = \left[
    \begin{array}{c}
      Y_{1} \\
      Y_{2} \\
      \vdots \\
      Y_{c} 
    \end{array}
  \right].
  \label{eq:data}
\end{equation}
Then, the $i$-th party has a partial dataset and the corresponding ground truth,
\begin{equation*}
  X_{i} \in \mathbb{R}^{n_i \times m}, \quad Y_i \in \mathbb{R}^{n_i \times \ell},
\end{equation*}
where $n = \sum_{i=1}^c n_i$.
\par
Here, we introduce the algorithm of the DC analysis for supervised learning of horizontal partitioned data \eqref{eq:data} and analyze its identifiability.
DC analysis is applicable to the datasets with partially common features \cite{mizoguchi2022application} and horizontal and vertical partitioned data \cite{imakura2021collaborative}.
\subsection{DC analysis}
In the practical operation strategy, the DC analysis is operated by two roles, namely {\it worker} and {\it master}.
The workers have the private dataset $X_{i}$ and the corresponding ground truth $Y_i$ and want to analyze them without sharing $X_{i}$.
The master supports to collaborative analysis.
\par
First, all workers generate the same anchor data $X^{\rm anc} \in \mathbb{R}^{r \times m}$, which is shareable data consisting of public data or dummy data that are randomly constructed.
A random matrix works well in general \cite{imakura2020data,imakura2021collaborative,imakura2021collaborative2}.
Then, each worker constructs intermediate representations,
\begin{align*}
  \widetilde{X}_{i} = f_{i}(X_{i}) \in \mathbb{R}^{n_i \times \widetilde{m}_{i}}, \quad
  \widetilde{X}_{i}^{\rm anc} = f_{i}(X^{\rm anc}) \in \mathbb{R}^{r \times \widetilde{m}_{i}},
\end{align*}
with a linear or nonlinear row-wise mapping function $f_{i}$ such as dimensionality reduction, with $\widetilde{m}_{i} < m$, including unsupervised methods \cite{pearson1901liii,he2004locality,maaten2008visualizing} and supervised methods \cite{fisher1936use,sugiyama2007dimensionality,li2017locality,imakura2019complex}.
For privacy and confidentiality concerns, the function $f_{i}$ should be set such that the original data $X_i$ and its intermediate representation $\widetilde{X}_i$ do not have (approximately) the same features.
Then, the intermediate representations are send to the master.
\par
At the master side, mapping function $g_i$ for the collaboration representation is constructed satisfying $g_i(\widetilde{X}_i^{\rm anc}) \approx g_{i'}(\widetilde{X}_{i'}^{\rm anc})$ $(i, i' = 1, 2, \dots, c)$ in some sense.
In practice, $g_i$ is set as a linear function $g_i(\widetilde{X}_i^{\rm anc}) = \widetilde{X}_j^{\rm anc} G_j$ with $G_i \in \mathbb{R}^{\widetilde{m}_i \times \widehat{m}}$ and is constructed using the following minimal perturbation problem:
\begin{equation*}
  \min_{E_i, G_i' (i = 1, 2, \dots, c), \|Z\|_{\rm F} = 1} \sum_{i=1}^c \| E_i \|_{\rm F}^2 \quad \mbox{s.t. } (\widetilde{X}_{i}^{\rm anc} + E_i) G_i' = Z.
\end{equation*}
where $\|\cdot\|_{\rm F}$ denotes the Frobenius-norm of a matrix.
This can be solved by a singular value decomposition (SVD)-based algorithm for total least squares problems.
Let 
\begin{equation}
  [\widetilde{X}^{\rm anc}_1, \widetilde{X}^{\rm anc}_2, \dots, \widetilde{X}^{\rm anc}_c] \approx  U_{\widehat{m}} \Sigma_{\widehat{m}} V_{\widehat{m}}^{\rm T}
  \label{eq:SVD}
\end{equation}
be the rank $\widehat{m}$ approximation based on SVD.
Then, the target matrix $G_i$ is obtained as follows:
\begin{equation}
  G_i = (\widetilde{X}_i^{\rm anc})^\dagger U_{\widehat{m}} C,
  \label{eq:g}
\end{equation}
where $\dagger$ denotes the Moore--Penrose inverse and $C \in \mathbb{R}^{\widehat{m} \times \widehat{m}}$ is a nonsingular matrix, for example, $C=I$ and $C=\Sigma_{\widehat{m}}$ are used in practice.
The collaboration representations are analyzed as a single dataset, that is,
\begin{equation*}
  Y \approx h(\widehat{X}), \quad
  \widehat{X}
  = [ \widehat{\bm x}_1, \widehat{\bm x}_2, \dots, \widehat{\bm x}_n]^{\rm T}
  = \left[
    \begin{array}{c}
      \widehat{X}_1 \\
      \widehat{X}_2\\
      \vdots \\
      \widehat{X}_{c}
    \end{array}
  \right] 
  = \left[
    \begin{array}{c}
      \widetilde{X}_1 G_1 \\
      \widetilde{X}_2 G_2 \\
      \vdots \\
      \widetilde{X}_c G_c
    \end{array}
  \right] 
  \in \mathbb{R}^{n \times \widehat{m}}
\end{equation*}
with the shared ground truth $Y_i$ using some supervised machine learning or the deep learning methods for constructing the model function $h$ of the collaboration representation $\widehat{X}$.
Functions $g_i$ and $h$ are returned to the $i$-th worker.
\par
Let $X_i^{\rm test} \in \mathbb{R}^{s_i \times m}$ be a test dataset of the $i$-th party.
For the prediction phase, the prediction result $Y_i^{\rm pred}$ of $X_i^{\rm test}$ is obtained by the following equation:
\begin{equation}
Y_i^{\rm pred} = h( g_i(f_i (X_{i}^{\rm test})))
\label{eq:prediction}
\end{equation}
through the intermediate and collaboration representations.
\par
The algorithm of the DC analysis is summarized in Algorithm~\ref{alg:naiveDC}, where $g_i$ is set by \eqref{eq:g}.
The DC analysis requires only three cross-institutional communications, Steps 1, 4, and 9 in Algorithm~\ref{alg:naiveDC}.
A major advantage is observed over federated learnings.
The most time-consuming parts of Algorithm~\ref{alg:naiveDC} are computing intermediate representations in each worker (Steps 2, 3), generating $G_i$ by \eqref{eq:SVD} and \eqref{eq:g} in the master (Step 6), and analyzing the collaboration representation in the master (Step 8).
Because the costs of computing intermediate representations and analyzing the collaboration representation are almost the same as that of the centralized analysis sharing the raw datasets, the main increase in computational complexity of the DC analysis relative to the centralized analysis is for generating $G_i$ by \eqref{eq:SVD} and \eqref{eq:g}.
\par
The DC analysis has the following double privacy layer for the protection of private data $X_{i}$:
\begin{itemize}
  \item No one can possess private data $X_{i}$ because $f_{i}$ is private under the protocol;
  \item Even if $f_{i}$ is stolen, private data $X_{i}$ is still protected regarding $\varepsilon$-DR privacy \cite{nguyen2020autogan} because $f_{i}$ is a dimensionality reduction function ($\widetilde{m}_i < m$),
\end{itemize}
see \cite{imakura2021accuracy} for more details.
\begin{algorithm*}[!t]
\caption{DC analysis}
\label{alg:naiveDC}
\small
\begin{algorithmic}
  \REQUIRE $X_{i} \in \mathbb{R}^{n_i \times m}$, $Y_i \in \mathbb{R}^{n_i \times \ell}$, and $X_i^{\rm test}$ individually
  \ENSURE $Y_i^{\rm pred}$ $(i = 1, 2, \dots, c)$.
  \STATE
  \STATE
  \begin{tabular}{rcll}
    & \multicolumn{2}{c}{ {\it Worker-side} $(i = 1, 2, \dots, c)$} & \\ \cmidrule{2-3}
    1:  & \multicolumn{2}{l}{Generate $X^{\rm anc}$ and share to all workers} & \\
    2:  & \multicolumn{2}{l}{Generate $f_i$} & \\
    3:  & \multicolumn{2}{l}{Compute $\widetilde{X}_i = f_i(X_i)$ and $\widetilde{X}^{\rm anc}_i =  f_i(X^{\rm anc})$} & \\
    4:  & \multicolumn{2}{l}{Share $\widetilde{X}_i, \widetilde{X}_i^{\rm anc}$, and $Y_i$ to master} & \\
    \\
    &   & \multicolumn{2}{c}{ {\it Master-side}}  \\ \cmidrule{3-4}
    5:  & \qquad \qquad \qquad \qquad $\searrow$ & \multicolumn{2}{l}{Obtain $\widetilde{X}_i, \widetilde{X}_i^{\rm anc}$, and $Y_i$ for all $i$}  \\
    6:  & & \multicolumn{2}{l}{Compute $G_i$ from $\widetilde{X}_{i}^{\rm anc}$ for all $i$ by \eqref{eq:SVD} and \eqref{eq:g}} \\
    7:  & & \multicolumn{2}{l}{Compute $\widehat{X}_{i} = \widetilde{X}_{i} G_i$ for all $i$, and set $\widehat{X}$} \\
    8:  & & \multicolumn{2}{l}{Analyze $\widehat{X}$ to obtain $h$ such that $Y \approx h(\widehat{X})$} \\
    9: & \qquad \qquad \qquad \qquad $\swarrow$ & \multicolumn{2}{l}{Return $G_i$ and $h$ to each worker} \\
    \\
    & \multicolumn{2}{c}{ {\it Worker-side} $(i = 1, 2, \dots, c)$} & \\ \cmidrule{2-3}
    10: & \multicolumn{2}{l}{Obtain $G_i$ and $h$} \\
    11: & \multicolumn{2}{l}{Compute $Y_i^{\rm pred} = h(f_i(X_i^{\rm test})G_i)$} \\
  \end{tabular}
\end{algorithmic}
\end{algorithm*}
\subsection{Analysis on identifiability of the intermediate representations}
Based on Definition~1, we analyze the identifiability of the intermediate representation $\widetilde{X}_i$ stored in the master.
\par
In DC analysis, for privacy and confidentiality concerns, function $f_{i}$ is constructed such that the original data $X_i$ and its intermediate representation $\widetilde{X}_i$ do not have (approximately) the same features.
Therefore, no identical features exist that are key to collate $\widetilde{X}_i$ with $X_i$.
By contrast, in the naive implementation (Algorithm~\ref{alg:naiveDC}), $\widetilde{X}_i$ and $X_i$ share the row index, that is the $k$-th row of $\widetilde{X}_i$ is collated with the $k$-th row of $X_i$, because $f_i$ is a row-wise function.
This difficulty can be simply solved by using a random permutation matrix $P_i \in \mathbb{R}^{n_i \times n_i}$ such that each worker computes $\widetilde{X}'_i = P_i f_i(X_i)$ and $Y'_i = P_i Y_i$ and erase $P_i$ before sharing $\widetilde{X}'_i$ and $Y'_i$ to the master.
Here, $P_i$ does not change the machine learning model.
\par
However, $\widetilde{X}'_i$ is still readily identifiable to $X_i$ for supervised learning.
Because intermediate representation is constructed by $\widetilde{\bm x}_k = f_i({\bm x}_k)$ and function $f_i$ should be stored in the worker for prediction phase \eqref{eq:prediction}.
Therefore, based on Proposition~1, $\widetilde{X}'_i$ is readily identifiable to $X_i$.
\section{Non-readily identifiable DC analysis}
In Section~4.2, we revealed that the shared intermediate representations of the conventional DC analysis are readily identifiable to the original data for supervised learning even with a random sample permutation.
Identifiability of the shared data is essential for analyzing data including personal information.
\par
To solve this difficulty, we propose a non-readily identifiable DC analysis to realize privacy-preserving analysis of horizontal partitioned multiple medical datasets sharing only non-readily identifiable intermediate representations without iterative cross-institutional communications.
Note that the proposed non-readily identifiable DC analysis is naturally extendable to the datasets with partially common features and horizontal and vertical partitioned data in the same manner as \cite{mizoguchi2022application} and \cite{imakura2021collaborative} for the conventional DC analysis, respectively.
\subsection{Basic concept}
In this study, we propose a non-readily identifiable DC analysis in which the intermediate representations are non-readily identifiable to the original data based on the following strategies:
\begin{itemize}
  \item The proposed method uses a random permutation for the sample index in each worker.
  \item The proposed method makes functions $f_i$ erasable based on the concept of interpretable DC analysis \cite{imakura2021interpretable}.
  \item The proposed method uses functions $f_i$ that cannot be reconstructed.
\end{itemize}
\subsection{Practical algorithm}
Each worker constructs a dimensionality-reduced function $f_i'$ that cannot be reconstructed.
Simple examples of $f_i'$ are as follows:
\begin{itemize}
  \item $f'_i$ is constructed as a dimensionality reduction function constructed from the raw data with a small perturbation $X_i + E_i$, where $E_i \in \mathbb{R}^{n_i \times m}$,
  \item $f'_i$ is constructed as $f'_i(X) = f_i(X) E_i$ with a dimensionality reduction function $f_i$ constructed from the raw data $X_i$ and a random matrix $E_i \in \mathbb{R}^{\widetilde{m}_i \times \widetilde{m}_i}$,
\end{itemize}
where entries of $E_i$ are generated by the hardware random number generator or pseudo random number generator with the CPU time as the seed.
Each worker constructs a permutation matrix $P_i$ as well as $E_i$ and subsequently computes the intermediate representation as follows:
\begin{equation*}
  \widetilde{X}_i' = P_i f_i'(X_i), \quad
  \widetilde{X}_i^{\rm anc} = f_i'(X_i^{\rm anc}), \quad
  {Y}_i' = P_i Y_i.
\end{equation*}
Next, each worker erases $P_i$ and $f_i'$ and share $\widetilde{X}_i, \widetilde{X}_i^{\rm anc}$ and $Y_i'$ to the master.
Function $f_i'$ and the matrix $P_i$ cannot be reconstructed even using the same program code.
In the case for classification problems, $P_i$ cannot be reconstructed from $Y_i$ and $Y_i'$ because each row of $Y_i$ and $Y_i'$ are not unique.
\par
At the master side, shared intermediate representations are transformed to the collaboration representations $\widehat{X}_i'$ in the same manner as the conventional DC analysis written in Section~3.1.1 and analyzed as a single dataset as follows:
\begin{equation*}
  Y' \approx h(\widehat{X}').
\end{equation*}
\par
To realize prediction without $f_i'$, based on the concept of the interpretable DC analysis, a prediction result $Y_i^{\rm anc}$ of the anchor data $X^{\rm anc}$ is computed as follows:
\begin{equation*}
  Y^{\rm anc}_i = h(\widehat{X}_i^{\rm anc}), \quad
  \widehat{X}^{\rm anc}_i = \widetilde{X}_i G_i \in \mathbb{R}^{r \times \widehat{m}}.
\end{equation*}
The prediction result $Y_i^{\rm anc}$ of $X^{\rm anc}$ is returned to the $i$-th worker. 
The prediction model $t_i$ is then constructed by some supervised machine learning or the deep learning methods from $X^{\rm anc}$ and $Y_i^{\rm anc}$, that is,
\begin{equation*}
  Y_i^{\rm anc} \approx t_i(X^{\rm anc}),
\end{equation*}
in each worker. 
For the prediction phase, the prediction result $Y_i^{\rm pred}$ of $X_i^{\rm test}$ is obtained by
\begin{equation*}
  Y_i^{\rm pred} = t_i (X_{i}^{\rm test}).
\end{equation*}
\begin{figure*}[!t]
\centering
\includegraphics[scale=0.35, bb = 0 150 1134 570]{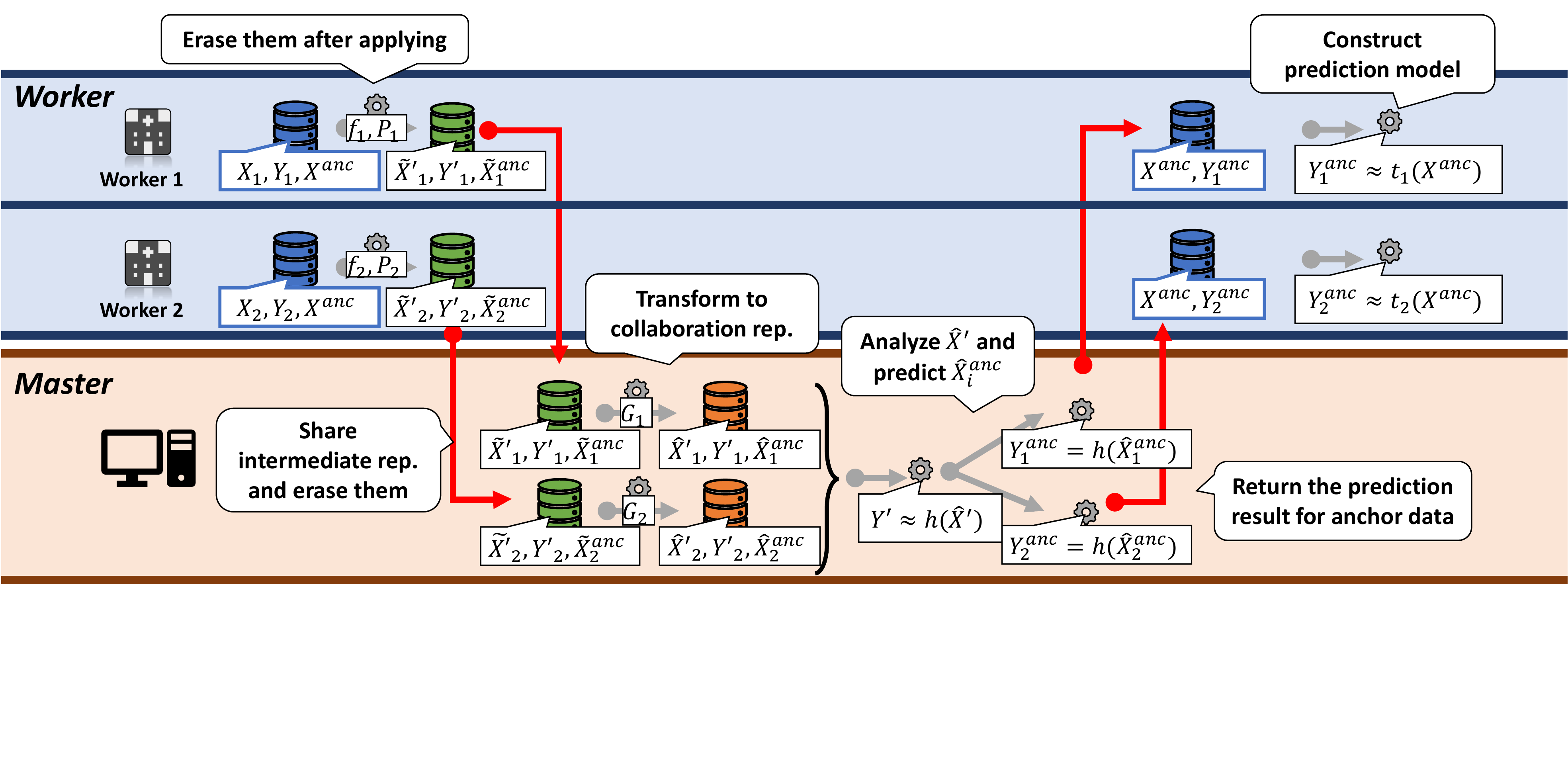}
\caption{The overview of non-readily identifiable DC analysis with $c=2$.}
\label{fig:algorithm}
\end{figure*}
\par
The algorithm of the non-readily identifiable DC analysis is summarized in Algorithm~\ref{alg:nonIDC} and Figure~\ref{fig:algorithm}.
The proposed non-readily identifiable DC analysis requires only three cross-institutional communications, namely Steps 1, 5, and 11 in Algorithm~\ref{alg:nonIDC} and has a double privacy layer for protection of private data $X_i$ as well as the conventional DC analysis described in Section~4.1.
Regarding computational complexity, the proposed non-readily identifiable DC analysis is more computationally intensive than the conventional DC analysis for analyzing anchor data $\widehat{X}$ to obtain the model $t_i$ (Step 13 in Algorithm~\ref{alg:nonIDC}).
\begin{algorithm*}[!t]
\caption{Non-readily identifiable DC analysis}
\label{alg:nonIDC}
\small
\begin{algorithmic}
  \REQUIRE $X_{i} \in \mathbb{R}^{n_i \times m}$, $Y_i \in \mathbb{R}^{n_i \times \ell}$, and $X_i^{\rm test}$ individually
  \ENSURE $Y_i^{\rm pred}$ $(i = 1, 2, \dots, c)$.
  \STATE
  \STATE
  \begin{tabular}{rcll}
    & \multicolumn{2}{c}{ {\it Worker-side} $(i = 1, 2, \dots, c)$} & \\ \cmidrule{2-3}
    1:  & \multicolumn{2}{l}{Generate $X^{\rm anc}$ and share to all workers} & \\
    2:  & \multicolumn{2}{l}{Generate $f_i'$ and $P_i$ that cannot be reconstructed} & \\
    3:  & \multicolumn{2}{l}{Compute $\widetilde{X}_i = P_if_i'(X_i), \widetilde{X}^{\rm anc}_i =  f_i'(X^{\rm anc})$, and $Y_i' = P_iY_i$} & \\
    4:  & \multicolumn{2}{l}{Erase $f_i'$ and $P_i$} & \\
    5:  & \multicolumn{2}{l}{Share $\widetilde{X}_i', \widetilde{X}_i^{\rm anc}$, and $Y_i'$ to master and erase them} & \\
    \\
    &   & \multicolumn{2}{c}{ {\it Master-side}}  \\ \cmidrule{3-4}
    6:  & \qquad \qquad \qquad \qquad $\searrow$ & \multicolumn{2}{l}{Obtain $\widetilde{X}_i', \widetilde{X}_i^{\rm anc}$, and $Y_i'$ for all $i$}  \\
    7:  & & \multicolumn{2}{l}{Compute $G_i$ from $\widetilde{X}_{i}^{\rm anc}$ for all $i$ by \eqref{eq:g}} \\
    8:  & & \multicolumn{2}{l}{Compute $\widehat{X}_{i}' = \widetilde{X}_{i}' G_i$ for all $i$, and set $\widehat{X}'$} \\
    9:  & & \multicolumn{2}{l}{Analyze $\widehat{X}'$ to obtain $h$ such that $Y' \approx h(\widehat{X}')$} \\
    10:  & & \multicolumn{2}{l}{Compute $Y_i^{\rm anc} = h(\widetilde{X}_i^{\rm anc} G_i)$} \\
    11: & \qquad \qquad \qquad \qquad $\swarrow$ & \multicolumn{2}{l}{Return $Y_i^{\rm anc}$ to each worker} \\
    \\
    & \multicolumn{2}{c}{ {\it Worker-side} $(i = 1, 2, \dots, c)$} & \\ \cmidrule{2-3}
    12:  & \multicolumn{2}{l}{Obtain $Y_i^{\rm anc}$} \\
    13:  & \multicolumn{2}{l}{Analyze ${X}^{\rm anc}$ to obtain $t$ such that $Y_i^{\rm anc} \approx t_i(X^{\rm anc})$} \\
    14: & \multicolumn{2}{l}{Compute $Y_i^{\rm pred} = t_i(X_i^{\rm test})$} \\
  \end{tabular}
\end{algorithmic}
\end{algorithm*}
\subsection{Analysis on the identifiability of intermediate representations}
Based on the protocol of the proposed non-readily identifiable DC analysis (Algorithm~\ref{alg:nonIDC}), we have the following expression:
\begin{itemize}
  \item No common features to collate $X_i$ and $\widetilde{X}_i'$, based on the property of the function $f_i'$.
  \item No common sample ID to collate $X_i$ and $\widetilde{X}_i'$, based on using a random permutation $P_i$ that cannot be reconstructed.
  \item No function $f_i'$ because $f_i'$ cannot be constructed and is erased before sharing $\widetilde{X}_i'$.
    Function $f_i'$ can be re-generated only by solving the optimization problem $\widetilde{X}^{\rm anc}_i = f'_i(X^{\rm anc})$ for $f'_i$ by the cooperation of $i$-th worker and master.
  \item No function $(f_i')^{-1}$ such that ${\bm x}_k = (f_i')^{-1}(\widetilde{\bm x}_k')$ because $f_i'$ is a dimensionality-reduced function.
\end{itemize}
Therefore, from Definition~1, the intermediate representation $\widetilde{X}_i'$ is non-readily identifiable to original data $X_i$.
\section{Experiments}
In this section, the efficiency of the proposed non-readily identifiable DC analysis in Algorithm~\ref{alg:nonIDC} ({\bf DC-proposed}) is evaluated and compared with the conventional DC analysis in Algorithm~\ref{alg:naiveDC} ({\bf DC-naive}).
We compared with the centralized analysis that shares the raw datasets ({\bf Centralized}) and the local analysis that uses only local dataset $X_{i}$ ({\bf Local}).
{\bf Centralized} are considered as an ideal case because the raw data cannot be shared in our target situation.
\subsection{General settings}
As dimensionality reduction method in the worker-side of DC analyses, {\bf DC-naive} used the principal component analysis (PCA) \cite{pearson1901liii} and {\bf DC-proposed} used 
\begin{equation*}
  f'_i(X_i) = X_i F_i, \quad F_i = B_i E_i^{(1)},
\end{equation*}
where $B_i \in \mathbb{R}^{m \times \widetilde{m}_i}$ is a projection matrix of PCA for $X_i$ and $E_i^{(1)} \in \mathbb{R}^{\widetilde{m}_i \times \widetilde{m}_i}$ is a random matrix whose entries are normal random numbers that cannot be reconstructed.
We set $\widehat{m} = \widetilde{m}_i$.
We used the ridge regression for analyzing the collaboration representation (Step 8 in Algorithm~\ref{alg:naiveDC} and Step 9 in Algorithm~\ref{alg:nonIDC}) and for prediction model (Step 13 in Algorithm~\ref{alg:nonIDC}).
\par
The anchor data $X^{\rm anc}$ is constructed based on a low-rank approximation-based approach introduced in \cite{imakura2021interpretable} as follows:
In each worker, the local anchor data $X_{i}^{\rm approx}$ is constructed by using a low-rank approximation of $X_{i}$ with random perturbation, that is,
\begin{equation*}
  X_{i}^{\rm approx} = X_{i}^{\rm TSVD} + \delta E_i^{(2)},
\end{equation*}
where $X_{i}^{\rm TSVD} \in \mathbb{R}^{n_i \times m}$ is a low-rank approximation based on the truncated SVD of $X_{i}$, $\delta=0.05$ is a perturbation parameter, and $E^{(2)}_i \in \mathbb{R}^{n_i \times m}$ is a random matrix whose entries are uniform random numbers in $[-1,1]$.
Sharing $X_{i}^{\rm approx}$ with all users, $n$ samples of anchor data are generated as $X^{\rm approx} = [(X^{\rm approx}_1)^{\rm T}, (X^{\rm approx}_2)^{\rm T}, \dots, (X^{\rm approx}_c)^{\rm T}]^{\rm T} \in \mathbb{R}^{n \times m}$.
Next, to generate $r$ samples of anchor data $X^{\rm anc}$, we applied an augmentation technique by linear combination if $r > n$, and we select $r$ samples randomly otherwise (that is, $r \leq n$).
We set $r= 2,000$ as the number of anchor data.
These settings are based on our preliminary experiments and not always the best choice in terms of performance.
\par
We used the ridge regression for {\bf Centralized} and {\bf Local}.
We set the ground truth $Y$ as a binary matrix whose $(i,j)$-th entry is 1 if the training data ${\bm x}_i$ are in class $j$ and $0$ otherwise.
This ground truth has been applied to various classification algorithms, including ridge regression and deep neural networks \cite{bishop2006pattern}.
\par
All numerical experiments were performed using MATLAB\footnote{Program codes are available from the corresponding author by reasonable request.}.
\subsection{Experiment I: proof-of-concept}
As a proof of concept of the proposed method, we used a simulated hospital dataset ``{\sf hospital.mat}'' from the MATLAB Statistics and Machine Learning Toolbox.
The dataset contains 100 samples with six features: sex (male or female), age, weight, smoker (true or false), systolic blood pressure (SBP), and diastolic blood pressure (DBP).
We generate a data matrix $X \in \mathbb{R}^{100 \times 4}$ with sex, age, weight, and smoker and set a label ${\bm y} = [y_1, y_2, \dots, y_{100}]^{\rm T} \in \mathbb{R}^{100}$ for high blood pressure as follows:
\begin{equation*}
  y_i = \left\{
    \begin{array}{ll}
      1 & \mbox{SBP} \geq 140 \mbox{ or } \mbox{DBP} \geq 80, \\
      0 & \mbox{otherwise}.
    \end{array}
  \right.
\end{equation*}
\subsubsection{Identifiability}
We discuss the identifiability of the proposed method.
We set two parties with randomly selected 10 samples as presented in Table~\ref{table:ex1_raw}.
According to the algorithm of the proposed method (Algorithm~\ref{alg:nonIDC}) with $\widetilde{m}_i = \widehat{m} = 3$, each party constructs and shares the intermediate representation $\widetilde{X}_i'$ to the master.
Table~\ref{table:ex1_ir} reveals the shared intermediate representations $\widetilde{X}'_1$ and $\widetilde{X}'_2$ with sample IDs.
We note that the sample IDs are not stored in actual operation.
Next, the intermediate representations are transformed to the collaboration representations $\widehat{X}'_i$ as presented in Table~\ref{table:ex1_dc}.
\begin{table}[t]
\footnotesize
\centering
\caption{Raw data $X_1$ and $X_2$ with sample IDs.}
\label{table:ex1_raw}
\begin{tabular}{ccccccccccc}
\toprule
\multicolumn{5}{c}{In Party 1 $(X_1)$ } & & \multicolumn{5}{c}{In Party 2 ($X_2$)} \\ \cmidrule{1-5} \cmidrule{7-11}
IDs& Sex & Age &Weight& Smoke & & IDs & Sex & Age &Weight& Smoke \\ \cmidrule{1-5} \cmidrule{7-11}
$ 1$& $1$ & $32$ & $183$ & $0$ & & $ 1$& $0$ & $44$ & $146$ & $1$  \\
$ 2$& $1$ & $39$ & $188$ & $0$ & & $ 2$& $0$ & $25$ & $114$ & $0$  \\
$ 3$& $1$ & $43$ & $163$ & $0$ & & $ 3$& $0$ & $37$ & $129$ & $0$  \\
$ 4$& $0$ & $38$ & $124$ & $0$ & & $ 4$& $0$ & $28$ & $123$ & $1$  \\
$ 5$& $1$ & $45$ & $181$ & $0$ & & $ 5$& $1$ & $25$ & $174$ & $0$  \\
$ 6$& $0$ & $32$ & $136$ & $0$ & & $ 6$& $1$ & $33$ & $180$ & $1$  \\
$ 7$& $0$ & $39$ & $117$ & $0$ & & $ 7$& $0$ & $28$ & $111$ & $0$  \\
$ 8$& $1$ & $45$ & $170$ & $1$ & & $ 8$& $1$ & $45$ & $172$ & $1$  \\
$ 9$& $1$ & $50$ & $172$ & $0$ & & $ 9$& $0$ & $36$ & $129$ & $0$  \\
$10$& $0$ & $44$ & $136$ & $1$ & & $10$& $0$ & $40$ & $137$ & $0$  \\
\bottomrule
\end{tabular}
\end{table}
\begin{table}[t]
\footnotesize
\centering
\caption{
Intermediate representations $\widetilde{X}'_1$ and $\widetilde{X}'_2$ with sample IDs. Sample IDs are not stored in actual operation. \\
{\bf Remark:} {\it The columns of the intermediate representations cannot be a key to collate raw data.}
}
\label{table:ex1_ir}
\begin{tabular}{crrrccrrr}
\toprule
\multicolumn{4}{c}{From Party 1} & & \multicolumn{4}{c}{From Party 2} \\ \cmidrule{1-4} \cmidrule{6-9}
IDs & $\widetilde{X}'_1(:,1)$ & $\widetilde{X}'_1(:,2)$ & $\widetilde{X}'_1(:,3)$ & & IDs &  $\widetilde{X}'_2(:,1)$ & $\widetilde{X}'_2(:,2)$ & $\widetilde{X}'_2(:,3)$ \\ \cmidrule{1-4} \cmidrule{6-9}
$ 5$ & $-220.48$ & $281.54$ & $-132.47$ & & $ 6$ & $163.34$ & $-145.69$ & $-139.68$  \\
$ 7$ & $-145.23$ & $192.60$ & $-76.83 $ & & $ 8$ & $164.24$ & $-140.06$ & $-133.82$  \\
$ 1$ & $-219.27$ & $270.20$ & $-145.82$ & & $ 9$ & $124.30$ & $-104.44$ & $-101.18$  \\
$ 2$ & $-226.92$ & $284.14$ & $-144.41$ & & $ 4$ & $115.00$ & $-100.18$ & $-95.27 $  \\
$ 4$ & $-153.01$ & $200.55$ & $-84.36 $ & & $ 2$ & $105.73$ & $-91.89 $ & $-89.22 $  \\
$10$ & $-168.93$ & $222.65$ & $-89.66 $ & & $ 1$ & $142.96$ & $-119.37$ & $-113.59$  \\
$ 3$ & $-199.22$ & $256.20$ & $-117.13$ & & $ 5$ & $153.42$ & $-139.48$ & $-135.81$  \\
$ 6$ & $-165.21$ & $209.60$ & $-101.05$ & & $10$ & $133.08$ & $-111.02$ & $-107.51$  \\
$ 9$ & $-211.47$ & $275.30$ & $-119.52$ & & $ 7$ & $105.16$ & $-89.69 $ & $-86.98 $  \\
$ 8$ & $-208.30$ & $267.56$ & $-121.20$ & & $ 3$ & $124.91$ & $-104.50$ & $-101.21$  \\
\bottomrule
\end{tabular}
\end{table}
\begin{table}[t]
\footnotesize
\centering
\caption{
Collaboration representations $\widehat{X}'_1$ and $\widehat{X}'_2$ with sample ID. Sample IDs are not stored in actual operation. \\
{\bf Remark:} {\it The columns of the collaboration representations cannot be a key to collate raw data.}
}
\label{table:ex1_dc}
\begin{tabular}{crrrccrrr}
\toprule
\multicolumn{4}{c}{From Party 1} & & \multicolumn{4}{c}{From Party 2} \\ \cmidrule{1-4} \cmidrule{6-9}
IDs & $\widehat{X}'_1(:,1)$ & $\widehat{X}'_1(:,2)$ & $\widehat{X}'_1(:,3)$ & & IDs &  $\widehat{X}'_2(:,1)$ & $\widehat{X}'_2(:,2)$ & $\widehat{X}'_2(:,3)$ \\ \cmidrule{1-4} \cmidrule{6-9}
$ 5$ & $-464.96$ & $0.26  $ & $0.97 $ &  & $ 6$ & $-453.51$ & $-14.49$ & $-0.72$  \\
$ 7$ & $-308.22$ & $13.40 $ & $0.56 $ &  & $ 8$ & $-443.80$ & $4.08  $ & $-0.64$  \\
$ 1$ & $-459.64$ & $-17.67$ & $0.88 $ &  & $ 9$ & $-334.40$ & $5.44  $ & $0.66 $  \\
$ 2$ & $-476.94$ & $-10.01$ & $0.94 $ &  & $ 4$ & $-314.17$ & $-2.38 $ & $-0.92$  \\
$ 4$ & $-324.08$ & $9.77  $ & $0.57 $ &  & $ 2$ & $-290.24$ & $-4.24 $ & $0.53 $  \\
$10$ & $-357.47$ & $14.64 $ & $-0.85$ &  & $ 1$ & $-381.22$ & $11.32 $ & $-0.73$  \\
$ 3$ & $-420.64$ & $3.51  $ & $0.91 $ &  & $ 5$ & $-432.83$ & $-23.98$ & $0.71 $  \\
$ 6$ & $-347.94$ & $-2.14 $ & $0.55 $ &  & $10$ & $-356.51$ & $8.13  $ & $0.72 $  \\
$ 9$ & $-447.45$ & $9.86  $ & $0.98 $ &  & $ 7$ & $-285.44$ & $0.73  $ & $0.55 $  \\
$ 8$ & $-439.05$ & $4.71  $ & $-0.55$ &  & $ 3$ & $-335.18$ & $6.77  $ & $0.67 $  \\
\bottomrule
\end{tabular}
\end{table}
\par
From the comparison of Tables~\ref{table:ex1_raw} and \ref{table:ex1_ir}, the columns of the intermediate representations are uncorrelated or weakly correlated with the features of the raw data (the Pearson correlation coefficients are less than 0.4), which indicates the columns of the intermediate representations cannot be a key to collate $X_i$.
Similarly, from the comparison of Tables~\ref{table:ex1_raw} and \ref{table:ex1_dc}, the columns of the collaboration representations cannot be a key to collate $X_i$.
\par
On the other hand, according to the algorithm, function $f'_i$ has been erased.
Even if the same program code is used to reconstruct $f'_i$, it will not be the same function because random matrix $E_i^{(1)}$ differs completely.
In Table~\ref{table:ex1_re}, the columns of the original and reconstructed $f'_1(X_1)$ for the same data $X_1$ are presented.
The reconstructed $f'_1(X_1)$ differs completely from the original and cannot be a key to collate $X_i$.
\par
Thus, from Definition~1, the intermediate and collaboration representations, $\widetilde{X}_i'$ and $\widehat{X}_i'$, are non-readily identifiable to original data $X_i$.
\begin{table}[t]
\footnotesize
\centering
\caption{
The 1st, 2nd, and 3rd columns of original and reconstructed $f'_1(X_1)$ for the same $X_1$ with sample IDs. \\
{\bf Remark:} {\it Reconstructed $f'_1(X_1)$ differs completely from original $f'_1(X_1)$.}
}
\label{table:ex1_re}
\begin{tabular}{crrrccrrr}
\toprule
 \multicolumn{4}{c}{Original $f'_1(X_1)$} & & \multicolumn{4}{c}{Reconstructed $f'_1(X_1)$} \\ \cmidrule{1-4} \cmidrule{6-9}
IDs  & 1st & 2nd & 3rd & & IDs & 1st & 2nd & 3rd  \\ \cmidrule{1-4} \cmidrule{6-9}
$ 1$ & $-219.27$ & $270.20$ & $-145.82$ & & $ 1$& $-3.12$ & $105.89$ & $38.18$  \\
$ 2$ & $-226.92$ & $284.14$ & $-144.41$ & & $ 2$& $1.35 $ & $107.70$ & $53.41$  \\
$ 3$ & $-199.22$ & $256.20$ & $-117.13$ & & $ 3$& $8.00 $ & $91.75 $ & $67.59$  \\
$ 4$ & $-153.01$ & $200.55$ & $-84.36 $ & & $ 4$& $10.02$ & $68.93 $ & $63.63$  \\
$ 5$ & $-220.48$ & $281.54$ & $-132.47$ & & $ 5$& $6.84 $ & $102.37$ & $68.68$  \\
$ 6$ & $-165.21$ & $209.60$ & $-101.05$ & & $ 6$& $3.79 $ & $77.31 $ & $47.37$  \\
$ 7$ & $-145.23$ & $192.60$ & $-76.83 $ & & $ 7$& $11.80$ & $64.49 $ & $67.32$  \\
$ 8$ & $-208.30$ & $267.56$ & $-121.20$ & & $ 8$& $8.54 $ & $96.37 $ & $70.78$  \\
$ 9$ & $-211.47$ & $275.30$ & $-119.52$ & & $ 9$& $11.89$ & $96.00 $ & $82.03$  \\
$10$ & $-168.93$ & $222.65$ & $-89.66 $ & & $10$& $12.80$ & $75.90 $ & $75.10$  \\
\bottomrule
\end{tabular}
\end{table}
\subsubsection{Recognition performance}
The recognition performance of the proposed method was evaluated.
We set 1--8 parties with randomly selected 10 samples for the training dataset.
We also set test dataset $X^{\rm test}$ with 20 samples.
\par
Recognition performance was evaluated by the area under the curve (AUC) of the receiver operating characteristic, and its average and standard error of 20 trials for each method are detailed in Figure~\ref{fig:hospital}.
The results reveal that the recognition performance of {\bf Centralized}, {\bf DC-naive} and {\bf DC-proposed} increase with the increase in the number of parties.
{\bf DC-proposed} reveals almost the same high values as that of {\bf DC-naive} with the exception of one party and is comparable to that of {\bf Centralized}.
For eight parties, {\bf Centralized}, {\bf DC-naive}, and {\bf DC-proposed} achieves 12 percentage point, 10 percentage point, and 9 percentage point performance improvements in the average AUC over {\bf Local}, respectively.
\begin{figure}[!t]
\centering
\includegraphics[scale=0.5, bb = 50 250 545 602]{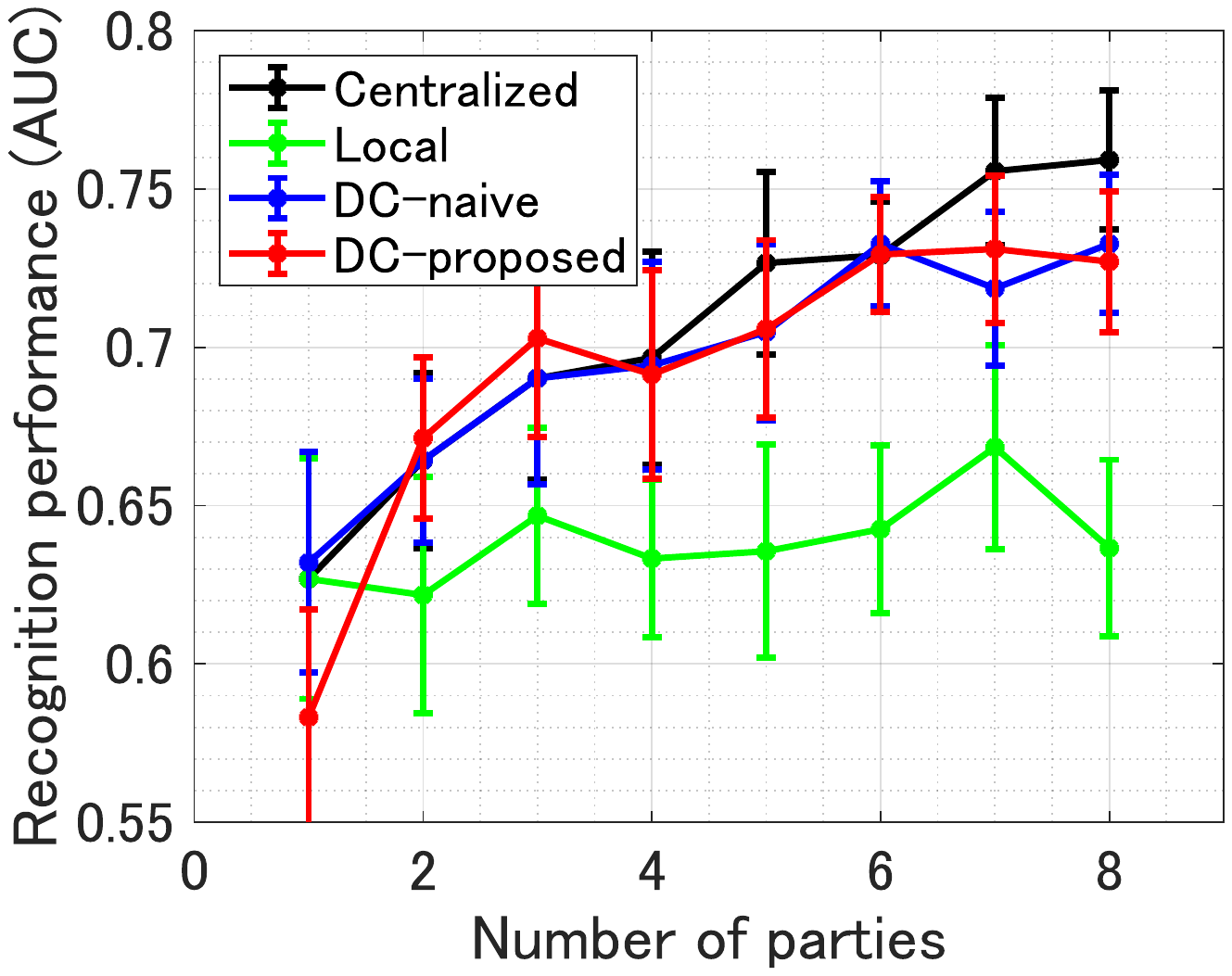}
\caption{Recognition performance (AUC) for the hospital dataset. \\
  {\bf Remark:} {\it {\bf DC-proposed} reveals almost the same high values as that of {\bf DC-naive} and is comparable to that of {\bf Centralized}.}
}
\label{fig:hospital}
\end{figure}
\subsection{Experiment II: performance evaluation for real-world datasets}
Here, the recognition performance for five datasets was evaluated as presented in Table~\ref{table:ex2_data}, obtained from the survival package of the R language.
For each dataset, a label ${\bm y} = [y_1, y_2, \dots$, $y_n]^{\rm T} \in \mathbb{R}^n$ was set using a survival time $t_i$ as presented in Table~\ref{table:ex2_data2}.
We set four parties with randomly selected 10 samples for the training dataset and set test dataset $X^{\rm test}$ with 20 samples.
\par
Average recognition performance (AUC) and standard error in 20 trials of methods are presented in Table~\ref{table:ex2_results}. 
The results demonstrated that the recognition performance of {\bf DC-proposed} reveals almost the same high values as that of {\bf DC-naive} and is comparable to that of {\bf Centralized}.
In some cases, {\bf DC-naive} and {\bf DC-proposed} outperform {\bf Centralized}, which may be derived from a performance improvement because of a dimensionality reduction method and a model reconstruction with a large number of anchor dataset $X^{\rm anc}$.
This will be investigated in the future.
\begin{table*}[!t]
  \caption{Datasets obtained from the survival package of the R language for Experiment II.}
  \label{table:ex2_data}
\footnotesize
\centering
\begin{tabular}{lrrl} \toprule
  \multicolumn{1}{c}{Name of dataset} & \multicolumn{1}{c}{$n$} & \multicolumn{1}{c}{$m$} & \multicolumn{1}{c}{Description} \\ \midrule
  {\sf colon} & 888 & 13 & Chemotherapy for Stage B/C colon cancer data \\
  {\sf kidney} & 76 & 6 & Kidney Catheter Data \\
  {\sf lung} & 167 & 7 & NCCTG Lung Catheter Data \\
  {\sf pbc} & 276 & 17 & Mayo Clinic Primary Biliary Cholangitis Data \\
  {\sf veteran} & 137 & 4 & Veterans' Administration Lung Cancer Study Data\\
\bottomrule
\end{tabular}
\end{table*}
\begin{table*}[!t]
  \caption{Label setting for Experiment II.}
  \label{table:ex2_data2}
\footnotesize
\centering
\begin{tabular}{lcc} \toprule
  \multicolumn{1}{c}{Name of dataset} & \multicolumn{1}{c}{condition for $y_i = 1$} & \multicolumn{1}{c}{rate of $y_i = 1$} \\ \midrule
  {\sf colon}   & $t_i \geq 1500$ & 0.50 \\
  {\sf kidney}  & $t_i \geq 100$  & 0.36 \\
  {\sf lung}    & $t_i \geq 400$  & 0.26 \\
  {\sf pbc}     & $t_i \geq 2000$ & 0.44 \\
  {\sf veteran} & $t_i \geq 100$  & 0.39 \\
\bottomrule
\end{tabular}
\end{table*}
\begin{table}[t]
\footnotesize
\centering
\caption{Recognition performance (AUC) for real-world datasets. \\
  {\bf Remark:} {\it {\bf DC-proposed} reveals almost the same high values as that of {\bf DC-naive} and is comparable to that of {\bf Centralized}.}
}
\label{table:ex2_results}
\begin{tabular}{lcccc}
\toprule
\multicolumn{1}{c}{Name of dataset} & \multicolumn{1}{c}{\bf Local} & \multicolumn{1}{c}{\bf Centralized} & \multicolumn{1}{c}{\bf DC-naive} & \multicolumn{1}{c}{\bf DC-Proposed} \\
\cmidrule(lr){1-1} \cmidrule(lr){2-2} \cmidrule(lr){3-3} \cmidrule(lr){4-4} \cmidrule(lr){5-5}
{\sf colon}   & $0.51 \pm 0.03$	& $0.63 \pm 0.02$	& $0.64 \pm 0.03$	& $0.63 \pm 0.03$ \\
{\sf kidney}  & $0.66 \pm 0.04$	& $0.72 \pm 0.02$	& $0.74 \pm 0.02$	& $0.74 \pm 0.02$ \\
{\sf lung}    & $0.50 \pm 0.04$	& $0.56 \pm 0.05$	& $0.53 \pm 0.04$	& $0.53 \pm 0.04$ \\
{\sf pbc}     & $0.61 \pm 0.03$	& $0.66 \pm 0.03$	& $0.71 \pm 0.02$	& $0.71 \pm 0.03$ \\
{\sf veteran} & $0.65 \pm 0.03$	& $0.69 \pm 0.02$	& $0.72 \pm 0.02$	& $0.72 \pm 0.02$ \\
\bottomrule
\end{tabular}
\end{table}
\section{Discussion}
Multi-source data fusion, in which multiple data sources are jointly analyzed to obtain improved information, has considerable research attention.
In case of multi-source data fusion on multiple medial institutions, data confidentiality and cross-institutional communication are major concerns for social implementation.
For analysis of data including personal information, identifiability of the shared data is also essential.
\par
Although DC analysis may be a suitable choice regarding data confidentiality and cross-institutional communication concerns, this study revealed that it is difficult to solve the identifiability concern based on the conventional DC framework.
This study then proposed the non-identifiable DC analysis that solves the identifiability concern while maintaining the strength of DC analysis.
The proposed method could become a breakthrough technology for future multi-source data fusion on multiple medical institutions.
\par
However, this study is based on the mathematical definition of identifiability (Definition~1 in Section~3).
The legal justification should be examined separately.
\section{Conclusions}
This study addressed challenges, such as data confidentiality, cross-institutional communication, and identifiability, of the multi-source data fusion of multiple medical datasets including personal information.
A recent DC framework was used to address the data confidentiality and cross-institutional communication concerns.
This study investigated the identifiability of DC analysis.
\par
A mathematical definition of the identifiability of the data (Definition~1 in Section~3) was introduced.
The results revealed that the shared intermediate representations of the conventional DC analysis are readily identifiable to the original data for supervised learning (Section 4.2).
Next, we proposed a non-readily identifiable DC analysis based on a random sample permutation, the concept of interpretable DC analysis, and usage of functions that cannot be reconstructed.
In this method, the shared intermediate representations are non-readily identifiable to the original data (Section 5).
\par
The proposed non-readily identifiable DC analysis exhibited a non-readily identifiability while maintaining a high recognition performance of the conventional DC analysis for medical datasets (Section~6).
For a hospital dataset, the proposed method exhibited a nine percentage point improvement regarding the recognition performance (AUC) over the local analysis.
The numerical experiments also demonstrated that the recognition performance of the proposed method was comparable to the centralized analysis that shares the raw datasets.
\par
In the future, we will further develop the algorithm and software for social implementation.
\section*{Acknowledgments}
This work was supported in part by the New Energy and Industrial Technology Development Organization (NEDO), Japan Science and Technology Agency (JST) (No. JPMJPF2017), and the Japan Society for the Promotion of Science (JSPS), Grants-in-Aid for Scientific Research (Nos. JP19KK0255, JP21H03451, JP22H00895, JP22K19767).

\bibliography{mybibfile}
\bibliographystyle{elsart-num-sort}

\end{document}